\documentclass[12pt]{article}
\usepackage{graphicx} 
\usepackage{mypkg}
\usepackage[colorinlistoftodos]{todonotes}
\usepackage[a4paper, total={6.3in, 8.8in}]{geometry}

\title{Point Cloud Classification via Deep Set Linearized Optimal Transport}

\author{Scott Mahan\footnote{Department of Mathematics, University of California, San Diego}, 
Caroline Moosm\"{u}ller\footnote{Department of Mathematics, University of North Carolina at Chapel Hill},
Alexander Cloninger$^*$\footnote{Halicioglu Data Science Institute, University of California, San Diego}
}

\date{}

\begin{document}

\maketitle

\begin{abstract}
 We introduce Deep Set Linearized Optimal Transport, an algorithm designed for the efficient simultaneous embedding of point clouds into an $L^2-$space. This embedding preserves specific low-dimensional structures within the Wasserstein space while constructing a classifier to distinguish between various classes of point clouds. Our approach is motivated by the observation that $L^2-$distances between optimal transport maps for distinct point clouds, originating from a shared fixed reference distribution, provide an approximation of the Wasserstein-2 distance between these point clouds, under certain assumptions.
To learn approximations of these transport maps, we employ input convex neural networks (ICNNs) and establish that, under specific conditions, Euclidean distances between samples from these ICNNs closely mirror Wasserstein-2 distances between the true distributions. Additionally, we train a discriminator network that attaches weights these samples and creates a permutation invariant classifier
to differentiate between different classes of point clouds. We showcase the advantages of our algorithm over the standard deep set approach through experiments on a flow cytometry dataset with a limited number of labeled point clouds.

\end{abstract}

\section{Introduction}

We consider the supervised learning problem of classifying probability measures $\mu_i$ on $\R^d$. Given training data $\{(\mu_i, y_i)\}_{i=1}^N$ with labels $y_i$, we want to learn a function $f: \PC(\R^d) \to \R$ such that $f(\mu_i) \approx y_i$. In practice, instead of probability measures $\mu_i$, our training set often consists of point clouds $X_i = \{x_k^{(i)}\}_{k=1}^{n_i}$ with $x_k^{(i)} \sim \mu_i$. This problem arises in a number of applications, for example, when grouping surveys by demographics \cite{AlexPeopleMovers}, topic labeling in a bag of words model \cite{Bagofwords}, and classification of patients from cell measurements \cite{CellSubpopulations, AlexAnisotropic, FlowCytometry}. As motivating example, we consider flow cytometry, which gathers up to dozens of measurements from thousands of cells at once for subsequent analysis \cite{FlowCytometry}. These features form a point cloud sampled from the distribution over the patient's cell population, allowing to train a point cloud classifier to identify healthy or sick patients. 

Classifying probability measures is more challenging than similar tasks on Euclidean vectors, because it requires a meaningful distance and feature space for probability measures or sets of points. The natural approach to this supervised learning problem is to embed the probability measures into a Euclidean space where traditional machine learning techniques can be applied.
Existing classifiers leverage the observation that any such embedding must be invariant to the order in which points are fed in. Many methods take the approach of embedding point clouds into a voxel occupancy grid, where features can be extracted using a 3D CNN \cite{3DShapeNets, VoxNet, 3dgan}. These methods have proven useful for classifying CAD models or LiDAR data, but they are restricted to 3D point clouds. In higher dimensions, some techniques embed $\mu_i$ using its moments \cite{Moments, chew2022manifold} or through kernel embeddings \cite{KernelMethods}, but these embeddings lose information about $\mu_i$ or distort the geometry of the data space. The Deep Sets method uses neural networks to learn arbitrary permutation-invariant feature representations of point clouds, but again provides no guarantee that these embeddings approximate distances between distributions \cite{DeepSets}. 


We propose using the linear optimal transport (LOT) embedding for distribution and point cloud classification. The LOT embedding maps a distribution to its optimal transport (OT) map relative to a fixed reference distribution \cite{Wang2013,park18,berman20,Holder2}. This maps a probability distribution into an (infinite-dimensional) Hilbert space $L^2(\R^d)$, in which we can train a classifier on the embedded data. The advantage of using the LOT embedding is an isometry property on distributions that are shifts or scalings of a fixed distribution, as shown in \cite{CarolineIsometry,park18,aldroubi20}. While distributions $\mu_i$ may not be exactly related by shifts or scalings in most problems, the LOT embedding also allows for shearings and $\varepsilon$-perturbations of such,  while still creating linearly separable convex classes in LOT space \cite{CarolineIsometry, VarunShearing}.
However, the LOT embedding also performs well experimentally on a number of classification tasks, even when the underlying probability measures are not related by shifts, scalings or shearings \cite{park18,wang11,basu14}.

In general, the optimal transport map between given probability measures or point clouds is not known a priori. To learn the optimal transport maps for each distribution to a reference, we can leverage Brenier's theorem, which states that optimal transport maps are gradients of convex functions \cite{Brenier}. This allows to implement the optimal transport problem using a continuous dual solver, which trains input-convex neural networks to learn the transport maps and their inverses \cite{NeuralOTSolvers}. While this requires a lot more computation than existing point cloud classifiers, it gives the added advantage of learning generative maps for the target distributions. 

Alongside learning the optimal transport maps, we simultaneously learn a softmax classifier on the LOT embedding space. In our implementation, we found it effective to alternate between training the LOT map networks and the classifier, since we can often achieve reliable classification accuracy before the OT maps have been fully learned. As a result, our model significantly outperforms DeepSets \cite{DeepSets} for point cloud classification.

\subsection{Related Work} 

The problem of 3D object detection from point clouds generated by LiDAR or other sensor data has been studied thoroughly. Many approaches to this problem use hand-designed features of the point clouds followed by a traditional machine learning classifier. These methods include range descriptors \cite{RangeDescriptor}, histogram descriptors \cite{HistDescriptors}, track classification \cite{TrackClass}, and other manual feature extractors \cite{ShapeDists, ShapeRecog}.

Other architectures for 3D object detection learn to extract features from raw data. First efforts at this problem apply standard 2D CNNs and RNNs to RGBD data, simply treating the depth as an additional input channel \cite{2.5DCNN, 2.5DCNN2, 2.5DCNN3}. However, this approach ignores the geometric information provided by the depth data. Attempts to better utilize this information map point cloud data to occupancy grids \cite{OccGrid}, which represent the probability of each voxel being occupied in a fixed 3D grid. The 3D ShapeNets \cite{3DShapeNets}, VoxNet \cite{VoxNet}, and 3D-GAN \cite{3dgan} architectures all propose different ways to leverage 3D CNNs on occupancy grid data. Many of these methods perform well for 3D object detection, but inherently cannot classify higher-dimensional point clouds.

In arbitrary dimensions, a simple way to embed a distribution into Euclidean space is by identifying it with its moments \cite{Moments}. This can be done with various types of moments; for example, \cite{chew2022manifold} uses moments of the manifold scattering transform of the distribution. However, these methods only use a finite number of moments and can lose important information given by higher-order moments. Another wide variety of methods (surveyed in \cite{KernelMethods}) maps distributions into a reproducing kernel Hilbert space so kernel machine learning techniques can be applied. In this case, geometric relationships between distributions are not preserved under certain transformations, such as shifting a distribution by a distance larger than the kernel's bandwidth.

One method for point cloud classification called DeepSets \cite{DeepSets} relies on the observation that any permutation-invariant classifier of a point cloud $X_i$ must be of the form $f(X_i) = \rho ( \hat{\E}_{x \in X_i} [\phi(x)] )$ for suitable transformations $\rho$ and $\phi$. A universal approximator is then trained to learn the feature representation $\phi$ and the classifier $\rho$ in the same way any deep network is trained.

We compare our method to DeepSets, and also note that learning the LOT embedding gives a generative model for the distributions and imposes more geometric regularization on the embedding space.



\section{Background: Optimal Transport}

In this paper, we are considering the problem of classifying probability measures $\mu \in \PC(\R^d)$. Typically, we consider measures with finite second moments, denoted by the space $\PC_2(\R^d)$. We use $\lambda$ to denote the Lebesgue measure on $\R^d$ and for $\sigma \in \PC_2(\R^d)$, we write $\sigma \ll \lambda$ if $\sigma$ is absolutely continuous with respect to $\lambda$. When $\sigma \ll \lambda$, there exists a density function $f_\sigma : \R^d \to \R$ such that
\begin{equation}
    \sigma(A) = \int_A f_\sigma(x) d\lambda(x)
\end{equation}
for any measureable set $A \subseteq \R^d$.

Given a map $T: \R^d \to \R^d$ and a measure $\sigma$, we define the \textit{pushforward measure} $T_\sharp \sigma$ by
\begin{equation}
    T_\sharp \sigma(A) = \sigma(T^{-1}(A)),
\end{equation}
assuming the preimage $T^{-1}(A)$ of $A$ under $T$ is measurable. Given two measures $\sigma$ and $\mu$, there are generally many maps $T$ such that $T_\sharp \sigma = \mu$. Monge's optimal transport problem seeks out the pushforward map which minimizes the cost of pushing $\sigma$ to $\mu$ \cite{OT}. Using the standard Euclidean distance on $\R^d$ as a cost function, the problem can be formulated as
\begin{equation} \label{eq:OT}
    \min_{T: \; T_\sharp \sigma = \mu} \int \|T(x) - x\|_2^2 d\sigma(x).
\end{equation}

A map $T$ achieving this minimum is called the \textit{optimal transport map} and is denoted by $T_\sigma^\mu$. In general, \eqref{eq:OT} may not have a solution. Kantorovich introduced a relaxation of \eqref{eq:OT} that takes the infimum over joint probability measures with marginals $\sigma$ and $\mu$:
\begin{equation} \label{eq:Kant}
    W_2(\sigma,\mu)^2 \defeq \inf_{\pi \in \Pi(\sigma,\mu)} \int_{\R^d \times \R^d} \|x-y\|_2^2 \, d\pi(x,y)
\end{equation}
where $\Pi(\sigma,\mu) = \{ \pi \in \PC(\R^d \times \R^d) : \pi(A \times \R^d)=\sigma(A), \pi(\R^d \times A)=\mu(A)$ for all measurable $A \subseteq \R^d \}$. The quantity $W_2(\sigma,\mu)$ is called the \textit{Wasserstein-2 distance} between $\sigma$ and $\mu$.

Problem \eqref{eq:Kant} is often referred to as the primal formulation of Wasserstein-2 distance. The Kantorovich duality theorem (see \cite{OT}, Theorem 5.10) states that
\begin{equation}
    W_2(\sigma,\mu)^2 = \sup_{(\psi,\varphi) \in \Phi} \left( \int \psi(x) d\sigma(x) + \int \varphi(y) d\mu(y) \right)
\end{equation}
where $\Phi = \{ (\psi,\varphi) \in L^1(\R^d,\sigma) \times L^1(\R^d,\mu) : \psi(x) + \varphi(y) \leq \|x-y\|_2^2 \,\,\, (\sigma \times \mu)\text{-a.e.} \}$. Since
\begin{equation*}
    \psi(x) + \varphi(y) \leq \|x-y\|_2^2 \iff \tfrac{1}{2}\|x\|_2^2 - \tfrac{1}{2}\psi(x) + \tfrac{1}{2}\|y\|_2^2 - \tfrac{1}{2}\varphi(y) \geq \langle x,y \rangle,
\end{equation*}
we can reparameterize $\tfrac{1}{2}\|\cdot\|_2^2 - \tfrac{1}{2}\psi(\cdot)$ as $\psi$ (and likewise for $\varphi$) to write
\begin{equation} \label{eq:reparam}
    W_2(\sigma,\mu)^2 = \tfrac{1}{2}\E_\sigma \|X\|_2^2 + \tfrac{1}{2}\E_\mu\|Y\|^2 - \inf_{(\psi,\varphi) \in \tilde{\Phi}} \left( \int \psi(x) d\sigma(x) + \int \varphi(y) d\mu(y) \right),
\end{equation} 
where $\tilde{\Phi} = \{ (\psi,\varphi) \in L^1(\R^d,\sigma) \times L^1(\R^d,\mu) : \psi(x) + \varphi(y) \geq \langle x,y \rangle \}$. If the infimum in \eqref{eq:reparam} is achieved, then Theorem 2.9 in \cite{Villani_OT_2003} says it can always be achieved by a pair of convex functions $(\psi, \psi^*)$, where $\psi^*(y) = \sup_x \big( \langle x,y \rangle - \psi(x) \big)$ is the convex conjugate of $\psi$. However, this formulation doesn't lend itself to stochastic optimization methods because of the conjugate function $\psi^*$. Instead, we will consider one final formulation of Wasserstein-2 distance, which is a minimax problem from Theorem 3.3 in \cite{OT_gradientbound}:
\begin{equation*}
    W_2(\sigma,\mu)^2 = \sup_{ \substack{\varphi \in \text{CV}(\mu) \\ \varphi^* \in L^1(\sigma)} } \inf_{\psi \in \text{CV}(\sigma)} \left( -\int \varphi(y) d\mu(y) - \int \big( \langle x, \nabla \psi(x) \rangle - \varphi(\nabla \psi(x)) \big) d\sigma(x) \right) + C_{\sigma,\mu} 
\end{equation*}
where $C_{\sigma,\mu} = \tfrac{1}{2}\E_\sigma \|X\|_2^2 + \tfrac{1}{2}\E_\mu\|Y\|^2$ and $\text{CV}(\sigma)$ is the set of convex functions in $L^1(\R^d,\sigma)$ (and likewise for $\mu$).
We will write this as
\begin{equation} \label{eq:V}
    W_2(\sigma,\mu)^2 = \sup_{ \substack{\varphi \in \text{CV}(\mu) \\ \varphi^* \in L^1(\R^d,\sigma)} } \inf_{\psi \in \text{CV}(\sigma)} V(\psi,\varphi) +  C_{\sigma,\mu}.
\end{equation}

Brenier's theorem \cite{Brenier} states that when $\sigma \ll \lambda$, a unique optimal transport map $T_\sigma^\mu$ exists. The optimal coupling is $\pi = (\text{Id}, T_\sigma^\mu)_\sharp \sigma$, and the optimal convex potential function $\psi: \R^d \to \R$ such that $T_\sigma^\mu = \nabla \psi$ exists and is unique up to an additive constant \cite{Brenier}.

\subsection{Linear Optimal Transport Embedding}

We define the linear optimal transport (LOT) embedding \cite{Wang2013,park18,berman20,Holder2}. Given a fixed reference measure $\sigma \in \PC_2(\R^d)$, we obtain a LOT embedding from $\PC_2(\R^d)$ into $L^2(\R^d, \sigma)$ defined by 
\begin{equation*}
    \mu \mapsto T_\sigma^\mu.
\end{equation*}
When $\sigma \ll \lambda$, this embedding is well-defined and we denote it by $F_\sigma$. 
The target space $L^2(\R^d, \sigma)$ is equipped with norm
\begin{equation*}
    \|f\|_\sigma^2 = \int \|f(x)\|_2^2 d\sigma(x).
\end{equation*}

The regularity of $F_\sigma$ can help us understand to what extent it preserves the geometry of $\PC_2(\R^d)$. \cite{Holder} shows that for a continuous curve of probability measures $(\mu_t)_{t \geq 0}$, the map $t \mapsto T_\sigma^{\mu_t}$ is $\tfrac{1}{2}$-H\"{o}lder continuous under some assumptions on $\sigma$ and $\mu_t$. 
Removing all assumptions on $\sigma$ and $\mu$, \cite{Holder2} proves that the Lebesgue LOT embedding $F_\lambda$ is a bi-H\"{o}lder continuous embedding of $\PC_2(\R^d)$ into $L^2(\R^d, \lambda)$; in particular, 
\begin{equation*}
    W_2(\sigma, \mu) \leq \|F_\lambda(\sigma) - F_\lambda(\mu)\|_\lambda \leq CW_2(\sigma,\mu)^{\frac{2}{15}}.
\end{equation*}
Both \cite{Holder} and \cite{Holder2} show that the H\"{o}lder exponent for $F_\sigma$ in general is not better than $\tfrac{1}{2}$.

Under much stronger assumptions on the probability measures, it can be shown that $F_\sigma$ is an isometry. To state these results, we introduce the concept of compatibility between an LOT embedding and certain transformations. We say that $F_\sigma$ is $\mu$-\textit{compatible} with a set of functions $\GC \subseteq L^2(\R^d,\mu)$ if $F_\sigma(g_\sharp \mu) = g \circ F_\sigma(\mu)$ for all $g \in \GC$. Letting
\begin{equation*}
    \GC \star \mu = \{ g_\sharp \mu : g \in \GC \},
\end{equation*}
we can examine the regularity of $F_\sigma$ on $\GC \star \mu$. 

We now define a few sets of functions of interest.
For $a \in \R^d$, let $S_a(x) = x + a$ and define $\SC = \{S_a : a \in \R^d\}$ as the set of shifts in $L^2(\R^d,\mu)$. For $c>0$, let $R_c(x) = cx$ and define $\RC = \{R_c : c>0\}$ as the set of scalings in $L^2(\R^d,\mu)$. Next let $\EC = \operatorname{span}\left(\SC \cup \RC\right)$. 
For some results, we will also require the functions in $\EC$ to be bounded. Hence, for $R>0$, we also define
\begin{equation*}
    \EC_{R} = \{ g \in \EC : \|g\|_\mu \leq R\}.    
\end{equation*}
Building on results of \cite{park18}, \cite{CarolineIsometry} shows that $F_\lambda$ is an isometry on $\EC \star \mu$. 

\begin{theorem} \label{thm:Fisometry}
    (Corollary 4.1 in \cite{CarolineIsometry}) Let $\mu \in \PC_2(\R^d)$ with $\mu \ll \lambda$. For $g_1,g_2 \in \EC$ and $\lambda$ the Lebesgue measure on a convex, compact subset of $\R^d$,
    \begin{equation*}
        \|F_\lambda(g_{1\sharp} \mu) - F_\lambda(g_{2\sharp} \mu)\|_\lambda = W_2(g_{1\sharp} \mu, g_{2\sharp} \mu). 
    \end{equation*}
\end{theorem}

Note that Theorem \ref{thm:Fisometry} as stated applies to the Lebesgue LOT embedding $F_\lambda$, but in general holds for $F_\sigma$.
The proof of Theorem \ref{thm:Fisometry} uses the fact that in arbitrary dimension $d \geq 1$, $F_\sigma$ is $\mu$-compatible with $\EC$. This can be used to show that $F_\sigma$ is an isometry on $\EC \star \mu$.
It follows that for convex $\HC \subset \EC$, the sets $F_\sigma(\HC \star \mu_i)$ and $F_\sigma(\HC \star \mu_j)$ are linearly separable in $L^2(\R^d,\lambda)$ as long as $W_2(\HC \star \mu_i, \HC \star \mu_j) > 0$. In other words, the LOT embedding can be used to classify two classes of distributions when each class consists of shifts and scalings of a fixed representative measure. This is a very powerful result, as a classifier based directly on the distances $W_2(\mu_i,\mu_j)$ requires solving $\binom{N}{2}$ expensive optimal transport problems, but using the LOT embedding allows us to solve for $N$ optimal transport maps and then classify based on $\binom{N}{2}$ less expensive distances in $L^2(\R^d, \lambda)$. 

However, classifying the LOT maps directly in $L^2(\R^d, \lambda)$ is still an infinite-dimensional problem, or very high-dimensional in the case of the discrete optimal transport problem implemented by \cite{NeuralOTSolvers}. Besides excessive computational and memory requirements, fitting a model on this high-dimensional feature space can lead to severe overfitting. With this in mind, our model approximates the optimal transport maps with input-convex neural networks and classifies each point cloud based on permutation-invariant statistics of the learned OT map.


\section{Input-Convex Neural Networks} \label{sec:ICNNs}

We begin by defining input-convex neural networks (ICNNs), introduced in \cite{ICNNs}. Their architecture is similar to that of a fully-connected network, but every hidden layer activation is a function of the previous layer \textit{and} the original input $x$.

\begin{definition}
    A $k$-layer \textit{input-convex neural network (ICNN)} is a function $h_\theta:\R^d \to \R^m$ defined by
    \begin{align*}
        z_1 &= g_1(W_0^{(x)}x + b_0); \\
        z_{i+1} &= g_i(W_i^{(x)}x + W_i^{(z)}z_i + b_i), \quad \quad i=1,\dots,k-1; \\
        h_\theta(x) &= z_k
    \end{align*}
    where all entries of $W_{1:k-1}^{(z)}$ are non-negative and all activation functions $g_i$ are convex and non-decreasing. 
\end{definition}

The requirements that all entries of $W_{1:k-1}^{(z)}$ are non-negative and all activation functions $g_i$ are convex and non-decreasing make $h_\theta$ convex. This is true because non-negative sums of convex functions and compositions of convex and non-decreasing convex functions are also convex. Just as standard fully-connected networks can approximate any continuous function, it turns out that ICNNs can approximate any convex function.

\begin{theorem} (Theorem 1 in \cite{ICNNua}) \label{thm:ua_ICNN}
    Assume $K \subseteq \R^d$ is compact and $f: \R^d \to \R^m$ is Lipschitz and convex. Then for every $\epsilon>0$, there exists an ICNN $h_\theta$ with ReLU activation functions such that
    \begin{equation*}
        \sup_{x \in K} \|f(x) - h_\theta(x)\|_\infty < \epsilon.
    \end{equation*}
\end{theorem}

The universal approximation theorem for ICNNs will be very useful for learning optimal transport maps. Recall by Brenier's Theorem \cite{Brenier} that the potential function for an OT plan is convex. Hence, we can approximate this potential function with an ICNN and in turn approximate the OT map by the gradient of our ICNN.  These have been used, most notably, to approximate optimal transport maps for a variety of tasks.  These include the original formulation of ICNNs \cite{ICNNs}, applications to optimal transport in generative modeling \cite{OT_gradientbound}, and a general survey and benchmarking of ICNNs as transport solvers \cite{NeuralOTSolvers}.
However, to characterize the accuracy of approximating OT maps with gradients of ICNNs, we will need some stronger requirements than regular convexity.

\subsection{Duality of Strong Convexity and Smoothness} \label{sec:alpha}

We now revisit some convexity and smoothness definitions. 

\begin{definition}
    A differentiable function $f: \Omega \to \R$ is \textnormal{$\alpha$-strongly convex} ($\alpha>0$) if for all $x,y \in \Omega$,
    \begin{equation*}
        f(x) \geq f(y) + \langle \nabla f(y), x-y \rangle + \frac{\alpha}{2} \|x-y\|^2.
    \end{equation*}
\end{definition}

\begin{definition}
    A function $f: \Omega \to \R$ is \textnormal{$\beta$-smooth} if for all $x,y \in \Omega$,
    \begin{equation*}
        \|\nabla f(x) - \nabla f(y)\| \leq \beta\|x-y\|.
    \end{equation*}
\end{definition}

In the context of OT map approximation, strong convexity is required to bound the gradient of the dual map. In fact, it turns out that strong convexity and smoothness are very closely related for dual functions.

\begin{lemma} \label{lem:duality}
    (Theorem 1 in \cite{duality}) Let $f: \Omega \to \R$ be differentiable and define $f^*(y) = \sup_x \big( \langle x,y \rangle - f(x) \big)$ to be the convex conjugate of $f$. Then $f$ is $\alpha$-strongly convex if and only if $f^*$ is convex and $\frac{1}{\alpha}$-smooth.
\end{lemma}

This duality between strong convexity and smoothness is useful for approximating OT maps because of the convexity assumptions in the dual formulation of the OT problem \eqref{eq:V}. Now we only need to assume that the OT map is Lipschitz in order to approximate it with an ICNN.

\begin{proposition} \label{prop:ICNN}
    Let $\sigma, \mu \in \PC_2(\R^d)$ with $\sigma, \mu \ll \lambda$ and $\epsilon>0$. Assume $T_\sigma^\mu$ is $\beta$-Lipschitz. Then there exists an ICNN $\hat{\psi}_\theta$ with
    \begin{equation*}
        \|T_\sigma^\mu - \nabla \hat{\psi}_\theta\|_\sigma \leq 4\beta \epsilon.
    \end{equation*}
\end{proposition}

\begin{proof} Let $\psi_\mu$ denote the optimal convex potential function satisfying $T_\sigma^\mu = \nabla \psi_\mu$ from Brenier's Theorem \cite{Brenier}. Note that
\begin{equation*}
    \|\nabla \psi_\mu(x) - \nabla \psi_\mu(y) \| = \|T_\sigma^\mu(x) - T_\sigma^\mu(y)\| \leq \beta\|x-y\|
\end{equation*}
and therefore $\psi_\mu$ is $\beta$-smooth. By Lemma \ref{lem:duality}, it follows that its conjugate $\psi_\mu^*$ is $\frac{1}{\beta}$-strongly convex. Moreover, $(\psi_\mu,\psi_\mu^*)$ solve the minimax formulation of the optimal transport problem \eqref{eq:V}. This means that
    \begin{equation} \label{eq:Valpha}
        \sup_{ \substack{\varphi \in \text{CV}^{1/\beta}(\mu) \\ \varphi^* \in L^1(\R^d,\sigma)} } \inf_{\psi \in \text{CV}(\sigma)} V(\psi,\varphi) = \sup_{ \substack{\varphi \in \text{CV}(\mu) \\ \varphi^* \in L^1(\R^d,\sigma)} } \inf_{\psi \in \text{CV}(\sigma)} V(\psi,\varphi) = V(\psi_\mu, \psi_\mu^*)
    \end{equation}
    since the left side of \eqref{eq:Valpha} takes the supremum over a subset $\text{CV}^{1/\beta}(\mu) \subset \text{CV}(\mu)$ but achieves the same value with $\varphi=\psi_\mu^*$. Note that $V(\psi,\varphi)$ is continuous, $\sigma,\mu \ll \lambda$ and $\psi_\mu^* - \tfrac{1}{2\beta}\|\cdot\|^2$ is convex. Since ICNNs are universal approximators of convex functions by Theorem \ref{thm:ua_ICNN}, there exist ICNNs $\hat{\psi}_\theta$ and $\hat{\varphi}_\omega$ such that $\hat{\psi}_\theta$ and $\hat{\varphi} \defeq \hat{\varphi}_\omega + \tfrac{1}{2\beta}\|\cdot\|^2$ solve \eqref{eq:Valpha} to an arbitrary degree of accuracy. Specifically, we can choose $\hat{\psi}_\theta$ and $\hat{\varphi}_\omega$ so that the minimization gap
    \begin{equation*}
        \epsilon_1(\hat{\psi}_\theta, \hat{\varphi}) \defeq V(\hat{\psi}_\theta, \hat{\varphi}) - \inf_{\psi \in \text{CV}(\sigma)} V(\psi, \hat{\varphi}) 
    \end{equation*}
    and the maximization gap
    \begin{equation*}
        \epsilon_2(\hat{\varphi}) \defeq  \sup_{ \substack{\varphi \in \text{CV}(\mu) \\ \varphi^* \in L^1(\R^d,\sigma)} } \inf_{\psi \in \text{CV}(\sigma)} V(\psi,\varphi) - \inf_{\psi \in \text{CV}(\sigma)} V(\psi, \hat{\varphi}) 
    \end{equation*}
    are arbitrarily small; say $\epsilon_1(\hat{\psi}_\theta, \hat{\varphi}), \epsilon_2(\hat{\varphi}) \leq \epsilon$. Since $\hat{\varphi}$ is $\frac{1}{\beta}$-strongly convex, we have
    \begin{equation*}
        \|T_\sigma^\mu - \nabla \hat{\psi}_\theta \|_\sigma = \|\nabla \psi_\mu - \nabla \hat{\psi}_\theta \|_\sigma \leq 2\beta (\epsilon_1(\hat{\psi}_\theta, \hat{\varphi}) + \epsilon_2(\hat{\varphi})) \leq 4\beta\epsilon
    \end{equation*}
    by Theorem 3.6 in \cite{OT_gradientbound}. Note that $\hat{\varphi} = \hat{\varphi}_\omega + \tfrac{1}{2\beta}\|\cdot\|^2$ approximating $\psi_\mu^*$ cannot be exactly represented by an ICNN because of the term $\tfrac{\alpha}{2}\|x\|^2$, but we are only worried about $\hat{\psi}_\theta$ being an ICNN in this result. Moreover, $\hat{\varphi}$ can still be \textit{approximated} by an ICNN since $\|\cdot\|^2$ is a convex function, or the term $\tfrac{1}{2\beta}\|x\|^2$ can simply be added to the output of the ICNN $\hat{\varphi}_\omega$ in practice.
\end{proof}

Proposition \ref{prop:ICNN} shows that any Lipschitz OT map can be approximated by the gradient of an ICNN to arbitrary accuracy, with the error depending linearly on the Lipschitz constant. Combined with the fact that the LOT embedding is nearly an isometry on certain distributions, this gives hope to our approach of classifying point clouds based on their output under trained ICNNs.


\section{Model Description and Regularity Results} \label{sec:model}

Our classification model using the LOT embedding draws inspiration from DeepSets \cite{DeepSets}, where it is noted that a point cloud classifier must be a function of permutation-invariant statistics $\hat{\E}_{x \in X_i} [\phi(x)]$. The paper \cite{DeepSets} aims for complete generality, allowing $\phi$ to be any feature representation. However, for general $\phi$ and two point clouds $X_i \sim \mu_i$ and $X_j \sim \mu_j$, 
there is no guarantee that
\begin{equation*}
    \| \hat{\E}_{x \in X_i} [\phi(x)] - \hat{\E}_{x \in X_j} [\phi(x)] \| \approx \delta (\mu_i,\mu_j)
\end{equation*}
for some metric $\delta$ on $\PC_2(\R^d)$.

We are interested in leveraging the regularity of the LOT embedding on certain classes of distributions to approximate the Wasserstein-2 distance between distributions. The most obvious approximation of $W_2(\mu_i,\mu_j)$ based on the LOT embedding is
\begin{equation*}
    W_{2,\sigma}^{\text{LOT}}(\mu_i,\mu_j) \defeq \|F_\sigma(\mu_i) - F_\sigma(\mu_j)\|_\sigma.
\end{equation*}
As discussed, Theorem \ref{thm:Fisometry} implies that $W_{2,\sigma}^{\text{LOT}}(\mu_i,\mu_j) = W_2(\mu_i,\mu_j)$ when $\mu_i$ and $\mu_j$ are shifts or scalings of a fixed distribution $\mu$. However, the LOT embedding $F_\sigma(\mu_i)$ cannot be known exactly from only finite samples of the measures.
\cite{LOTWassmap} provides high-probability guarantees of the quality of the LOT embedding when the optimal transport maps are approximated from finite samples. These results hold when solving the discrete OT problem via barycentric projection \cite{OT_rates} or using Sinkhorn's algorithm to solve an entropic regularized version of the problem \cite{Sinkhorn}. Each of these methods to approximate $W_2(\mu_i,\mu_j)$ can be thought of as linearized Wasserstein-2 distances, since the distances are actually computed in a Hilbert space.

We provide similar guarantees of the LOT embedding quality using ICNNs to approximate the OT maps. Given ICNNs $\hat{\psi}_{\theta_i}$ and $\hat{\psi}_{\theta_j}$ approximating the potential functions of $T_\sigma^{\mu_i}$ and $T_\sigma^{\mu_j}$, we define
\begin{equation*}
    W_{2,\sigma}^{\textnormal{LOT,NN}}(\mu_i,\mu_j; \hat{\psi}) \defeq \| \nabla \hat{\psi}_{\theta_i} - \nabla \hat{\psi}_{\theta_j} \|_\sigma = \left( \int \|\nabla \hat{\psi}_{\theta_i}(x) - \nabla \hat{\psi}_{\theta_j}(x) \|_2^2 d\sigma(x) \right)^{1/2}.
\end{equation*}
Also given a sample $\{X_k\}_{k=1}^n$ with $X_k \sim \sigma$, we define
\begin{equation*}
    \widehat{W}_{2,\sigma}^{\textnormal{LOT,NN}}(\mu_i,\mu_j; \hat{\psi}) \defeq \left( \frac{1}{n} \sum_{k=1}^n \left\| \nabla \hat{\psi}_{\theta_i}(X_k) - \nabla \hat{\psi}_{\theta_j}(X_k) \right\|^2 \right)^{1/2}
\end{equation*}
which is ultimately the quantity we use to approximate $W_2(\mu_i,\mu_j)$ in practice. As our main result below, we show that there are ICNNs such that $\widehat{W}_{2,\sigma}^{\textnormal{LOT,NN}}(\mu_i,\mu_j; \hat{\psi})$ approximates $W_2(\mu_i,\mu_j)$ to high accuracy with high probability when $\mu_i$ and $\mu_j$ are shifts or scalings of the same distribution.

\begin{theorem} \label{thm:main}
    Let $\sigma,\mu \in \PC_2(\R^d)$ with $\sigma, \mu \ll \lambda$ and $\supp(\mu) \subset B(0,1)$. Let $R,\epsilon>0$. Assume $T_\sigma^\mu$ is $\beta$-Lipschitz. Then for $g_1, g_2 \in \EC_{R}$, there exist ICNNs $\hat{\psi}_{\theta_1}$ and $\hat{\psi}_{\theta_2}$ satisfying
    \begin{equation*}
        \left| W_2(g_{1\sharp}\mu,g_{2\sharp}\mu) - \widehat{W}_{2,\sigma}^{\textnormal{LOT,NN}}(g_{1\sharp}\mu,g_{2\sharp}\mu; \hat{\psi}) \right| \leq 8\beta\epsilon + \frac{(4\beta\epsilon+R)^2}{R} \sqrt{\frac{\log(2/\delta)}{2n}}
    \end{equation*}
    with probability at least $1-\delta$. 
\end{theorem}

\textbf{Remark:} Theorem \ref{thm:main} as stated applies to exact shifts and bounded scalings of a fixed reference measure -- i.e., $g_1, g_2 \in \EC_{R}$. We note that perturbations of shifts, scalings, and certain shearings are also allowed by the results of \cite{CarolineIsometry,VarunShearing}. In particular, if we allow maps $h_1,h_2$ with $\|h_i-g_i\|_2 \leq \gamma$, then the bound in Theorem \ref{thm:main} incurs an additional cost of $c\gamma^\alpha$, where $\tfrac{2}{15}\leq \alpha \leq \tfrac{1}{2}$ according to the regularity properties of the target measure $\mu$.

\begin{proof}
    By the triangle inequality,
    \begin{align*}
        \Big| W_2(g_{1\sharp}\mu,g_{2\sharp}\mu) - \widehat{W}_{2,\sigma}^{\textnormal{LOT,NN}} \Big. &\Big.(g_{1\sharp}\mu,g_{2\sharp}\mu; \hat{\psi}) \Big| \leq \left| W_2(g_{1\sharp}\mu,g_{2\sharp}\mu) - W_{2,\sigma}^{\textnormal{LOT}}(g_{1\sharp}\mu,g_{2\sharp}\mu) \right| \\[+4mm] 
        &+ \underbrace{ \left| W_{2,\sigma}^{\textnormal{LOT}}(g_{1\sharp}\mu,g_{2\sharp}\mu) - W_{2,\sigma}^{\textnormal{LOT,NN}}(g_{1\sharp}\mu,g_{2\sharp}\mu; \hat{\psi}) \right| }_{\text{(a) ICNN approx. error}} \\
        &+ \underbrace{ \left| W_{2,\sigma}^{\textnormal{LOT,NN}}(g_{1\sharp}\mu,g_{2\sharp}\mu; \hat{\psi}) - \widehat{W}_{2,\sigma}^{\textnormal{LOT,NN}}(g_{1\sharp}\mu,g_{2\sharp}\mu; \hat{\psi}) \right| }_{\text{(b) discretized $\sigma$ sampling error}}. 
    \end{align*}
    Since $g_1, g_2 \in \EC_{R}$, we have $W_{2,\sigma}^{\textnormal{LOT}}(g_{1\sharp}\mu,g_{2\sharp}\mu) = W_2(g_{1\sharp}\mu,g_{2\sharp}\mu)$ by Theorem \ref{thm:Fisometry}. We will now bound the other two terms.
    \begin{enumerate}[label=(\alph*)]

        \item Note that
        \begin{align*}
            W_{2,\sigma}^{\textnormal{LOT}}(g_{1\sharp}\mu,g_{2\sharp}\mu) = \|T_\sigma^{g_{1\sharp}\mu} - T_\sigma^{g_{2\sharp}\mu} \|_\sigma \leq \|T_\sigma^{g_{1\sharp}\mu} &- \nabla \hat{\psi}_
            {\theta_1} \|_\sigma + \|\nabla \hat{\psi}_
            {\theta_1} - \nabla \hat{\psi}_
            {\theta_2} \|_\sigma \\
            &+ \|\nabla \hat{\psi}_
            {\theta_2} - T_\sigma^{g_{2\sharp}\mu} \|_\sigma \\
            = \|T_\sigma^{g_{1\sharp}\mu} &- \nabla \hat{\psi}_
            {\theta_1} \|_\sigma + \|T_\sigma^{g_{2\sharp}\mu} - \nabla \hat{\psi}_
            {\theta_2} \|_\sigma \\
            &+ W_{2,\sigma}^{\textnormal{LOT,NN}}(g_{1\sharp}\mu,g_{2\sharp}\mu; \hat{\psi}).
        \end{align*}
        Since $T_\sigma^\mu$ is $\beta$-Lipschitz, by Proposition \ref{prop:ICNN} $\hat{\psi}_{\theta_i}$ can be chosen so that
        \begin{equation*}
            \|T_\sigma^{g_{i\sharp}\mu} - \nabla \hat{\psi}_
            {\theta_i} \|_\sigma \leq 4\beta\epsilon
        \end{equation*}
        for $i=1,2$. It follows that
        \begin{equation*}
            \left| W_{2,\sigma}^{\textnormal{LOT}}(g_{1\sharp}\mu,g_{2\sharp}\mu) - W_{2,\sigma}^{\textnormal{LOT,NN}}(g_{1\sharp}\mu,g_{2\sharp}\mu; \hat{\psi}) \right| \leq 8\beta\epsilon. 
        \end{equation*}

        \item For this bound, we will apply McDiarmid's inequality to $\widehat{W}_{2,\sigma}^{\textnormal{LOT,NN}}(\mu_i,\mu_j; \hat{\psi})^2$. Define
        \begin{equation*}
            f(X_1,\dots,X_n) = \frac{1}{n} \sum_{k=1}^n \left\| \nabla \hat{\psi}_{\theta_i}(X_k) - \nabla \hat{\psi}_{\theta_j}(X_k) \right\|^2 = \widehat{W}_{2,\sigma}^{\textnormal{LOT,NN}}(\mu_i,\mu_j; \hat{\psi})^2.
        \end{equation*}
        Let $X=(X_1,\dots,X_{k-1},X_k,X_{k+1},\dots,X_n)$ and $X' = (X_1,\dots,X_{k-1},X_k',X_{k+1},\dots,X_n)$ differ in only one argument. We then have
        \begin{equation*}
            f(X) - f(X') = \frac{1}{n} \left\| \nabla \hat{\psi}_{\theta_i}(X_k) - \nabla \hat{\psi}_{\theta_j}(X_k) \right\|^2 - \frac{1}{n} \left\| \nabla \hat{\psi}_{\theta_i}(X_k') - \nabla \hat{\psi}_{\theta_j}(X_k') \right\|^2.
        \end{equation*}
        Next, we have
        \begin{align*}
            0 \leq \left\| \nabla \hat{\psi}_{\theta_i}(X_k) - \nabla \hat{\psi}_{\theta_j}(X_k) \right\| \leq \| \nabla \hat{\psi}_{\theta_i}(X_k) &- \nabla \psi_{\mu_i}(X_k) \| + \left\| \nabla \psi_{\mu_i}(X_k) - \nabla \psi_{\mu_j}(X_k) \right\| \\
            &+ \left\| \nabla \psi_{\mu_j}(X_k) - \nabla \hat{\psi}_{\theta_i}(X_k) \right\|
        \end{align*}
        by the triangle inequality. As in the proof of Proposition \ref{prop:ICNN}, we can choose $\hat{\psi}_{\theta_i}$ such that its dual is $\frac{1}{\beta}$-strongly convex, and therefore
        \begin{equation*}
            \| \nabla \hat{\psi}_{\theta_i}(X_k) - \nabla \psi_{\mu_i}(X_k) \| \leq 4\beta \epsilon
        \end{equation*}
        by Theorem 3.6 in \cite{OT_gradientbound} (and likewise for $\hat{\psi}_{\theta_j}$). For the middle term, we have
        \begin{equation*}
            \left\| \nabla \psi_{\mu_i}(X_k) - \nabla \psi_{\mu_j}(X_k) \right\| = \left\| T_\sigma^{\mu_i}(X_k) - T_\sigma^{\mu_j}(X_k) \right\| \leq 2R
        \end{equation*}
        since $\supp(\mu_i) \subset B(0,R)$. Putting all of these inequalities together gives
        \begin{equation*}
            |f(X)-f(X')| \leq \frac{1}{n}(8\beta\epsilon + 2R)^2.
        \end{equation*}
        By McDiarmid's inequality, it follows that
        \begin{equation*}
            \Pb\left( \left| \widehat{W}_{2,\sigma}^{\textnormal{LOT,NN}}(\mu_i,\mu_j; \hat{\psi})^2 - W_{2,\sigma}^{\textnormal{LOT,NN}}(\mu_i,\mu_j; \hat{\psi})^2 \right| \geq t \right) \leq 2\exp\left( - \frac{nt^2}{8(4\beta\epsilon+R)^4} \right).
        \end{equation*}
        Now, let $a=W_{2,\sigma}^{\textnormal{LOT,NN}}(\mu_i,\mu_j; \hat{\psi})$ and $b=\widehat{W}_{2,\sigma}^{\textnormal{LOT,NN}}(\mu_i,\mu_j; \hat{\psi})$. Since $\supp(\mu_i) \subset B(0,R)$, we have $a,b \leq 2R$. If $a+b=0$ then $a=b=0$ and we are done. Otherwise,
        \begin{equation*}
            |a-b| = \frac{|a^2-b^2|}{a+b} \geq \frac{|a^2-b^2|}{4R}
        \end{equation*}
        and therefore
        \begin{equation*}
            \Pb\left( \left| \widehat{W}_{2,\sigma}^{\textnormal{LOT,NN}}(\mu_i,\mu_j; \hat{\psi}) - W_{2,\sigma}^{\textnormal{LOT,NN}}(\mu_i,\mu_j; \hat{\psi}) \right| \geq t \right) \leq 2\exp\left( - \frac{2nR^2t^2}{(4\beta\epsilon+R)^4} \right)
        \end{equation*}
        by our result from McDiarmid's inequality. Solving $\delta = 2\exp\left( - \frac{2nR^2t^2}{(4\beta\epsilon+R)^4} \right)$ gives the final result.
    \end{enumerate}
\end{proof}

Finally, Theorem \ref{thm:main} shows that we can use ICNNs to approximate the Wasserstein distance between shifts and scalings of a fixed distribution to arbitrary accuracy with high probability. In the next section, we discuss how to increase the separability between classes for point cloud classification.

\subsection{Increasing Class Separation}

Theorem \ref{thm:main} shows that we can nearly isometrically embed distributions using ICNNs.  However, there also exist a number of applications where the point clouds $X_i$ are paired with a class label $y_i$, and one requires a classifier to go along with the embedding.  
From \cite{CarolineIsometry}, we know that if the classes are of the form $\{(g_\sharp \mu_i, y_i)\}_{g\in \mathcal{E}_R}$ for $y_i \in \{-1,1\}$, one can build a classifier $f:\mathbb{R}^d\rightarrow \mathbb{R}^d$ such that
\begin{align*}
\begin{cases}
    \int \langle f(x),  T_\sigma^{\nu_j}(x) \rangle d\sigma(x) > 0 & y_j=1\\
    \int \langle f(x),  T_\sigma^{\nu_j}(x) \rangle d\sigma(x) < 0 & y_j=-1,
\end{cases}
\end{align*}
for some test measure $(\nu_j, y_j)$ generated by the same process as the training set.

Given that we have strong approximations $\nabla \hat{\psi}_{\theta_i}$ of $T_\sigma^{\mu_i}$, we can similarly learn a weight function $W_\phi:\mathbb{R}^d\rightarrow\mathbb{R}^d$ and build a classifier
\begin{equation} \label{eq:fhat}
    \rho \Big( \tfrac{1}{n} \sum_{k=1}^n \langle W_\phi(x_k), \nabla \hat{\psi}_{\theta_i}(x_k) \rangle \Big)
\end{equation}
where $\rho$ is a sigmoid or softmax for classification. 
Note that there are no restrictions on $W_\phi$ and it can be trained as any other deep network.  In practice, we alternate training an epoch of the neural OT maps and training an epoch of the weights.  One significant benefit of this approach is that the weights may learn a high accuracy classifier before the neural OT maps have completely converged, which leads to early stopping.  Note that the inner product form of \eqref{eq:fhat} gives the needed permutation invariance required to have a distribution classifier \cite{DeepSets}.

\section{Experiments}

In this section, we compare DeepSets \cite{DeepSets} and the LOT embedding for point cloud classification. Since the optimal transport maps for $\mu_i$ are not known a priori, we begin by learning those before training our classifier. Recall that $T_\sigma^{\mu_i} = \nabla \psi_i$ for some convex function $\psi_i: \R^d \to \R$. \cite{NeuralOTSolvers} evaluates several neural dual OT solvers (i.e., those that approximate potentials $\psi$ and $\varphi$ for $T_\sigma^\mu$ and $T_\mu^\sigma$ using various neural network architectures and optimization schemes). In our implementation, we employ a non-minimax solver proposed in \cite{W2GAN} which approximates the potential functions $\psi_i$ with ICNNs to guarantee the convexity required by Brenier's theorem. We denote the approximation of $\psi_i$ by $\hat{\psi}_{\theta_i}$, where $\theta_i$ denotes the parameters of the corresponding neural network.

One difficulty of this approach is that it requires training the transport maps $\{\nabla \hat{\psi}_{\theta_i}\}_{i=1}^N$ simultaneously. However, since the LOT embedding uses the same source distribution $\sigma$ for each target distribution, we can take a batch samples $X = \{x_k\}_{k=1}^n$ from $\sigma$ and $X_i = \{x_k^{(i)}\}_{k=1}^n$ from each $\mu_i$ to train all OT maps on this batch.

\subsection{AML Classification}

We compare the methods on a dataset of patient cell measurements to classify whether patients have acute myeloid leukemia (AML). The dataset includes 359 total patients, 43 of which have AML. We randomly select 50\% of the AML-positive patients and add twice as many negative patients to the training set. This ensures a 2:1 class ratio in the dataset, so we can worry less about the class imbalance of the data as a whole. Of the remaining patients, we randomly select 10\% from each class for the validation data and leave the rest as testing data.

Each patient has data from about 30,000 cells, with 7 physical and chemical measurements taken from each cell via flow cytometry \cite{FlowCytometry}. After eliminating cells with missing or incomplete measurements, we randomly select 1,000 cells per patient as a point cloud representing that patient. In other words, we have a collection of distributions $\{\mu_i\}_{i=1}^{359}$ over the patient cell measurements, and from each $\mu_i$ we sample $X_i = \{x_k^{(i)}\}_{k=1}^{1000}$. Our goal is to find a classifier $\hat{f}$ of the form \eqref{eq:fhat} minimizing
\begin{equation*}
    \frac{1}{N} \sum_{i=1}^N \ell(\hat{f}(X_i),y_i)
\end{equation*}
where $\ell(p,y) = -[y \ln(p) + (1-y)\ln(1-p)]$ is the binary cross-entropy loss function.

\begin{table}[] 
\begin{center}
\begin{tabular}{|l|l|l|l|l|}  \hline
          & DeepSets & DeepSets Bagging & LOT Classifier & LOT resample x10 \\ \hline
precision & 0.17 & 0.23        & 0.67        & \textbf{0.67} \\ \hline
recall    & 0.33 & 0.44         & 0.44        & \textbf{0.67} \\ \hline
accuracy  & 0.84 & 0.86        & 0.95        & \textbf{0.96} \\ \hline
\end{tabular}
\centering
\caption{AML data test metrics.} \label{table:AML}
\end{center}
\end{table}

For each method, we compare the precision, recall, and accuracy metrics across the test data. Using the natural decision rule $\hat{y}_i = \bm{1}\{\hat{f}(X_i)>0.5\}$, these metrics can be defined as
\begin{equation*}
    \textnormal{precision}=\frac{TP}{TP+FP} \quad \quad \textnormal{recall}=\frac{TP}{TP+FN} \quad \quad \textnormal{accuracy}=\frac{TN+TP}{N}
\end{equation*}
where $TP=\textnormal{true positives}$, $FP=\textnormal{false positives}$, $FN=\textnormal{false negatives}$, and $TN=\textnormal{true negatives}$.

Table \ref{table:AML} shows test metrics for each method. For the first column, we applied one DeepSets model trained for 1000 epochs. DeepSets bagging refers to training 10 DeepSets models and averaging their predicted probabilities to make a classification. The third column shows the results for our LOT classifier, which alternately trained ICNNs for each OT map for 10 epochs followed by trainer a softmax classifier for 10 epochs until a total of 1000 epochs was reached. Because our LOT classifier is based on the OT map pushforward data $\hat{\psi}_{\theta_i}(\{x_k\}_{k=1}^{100})$ where $x_k \sim \sigma$, we can easily take another sample from $\sigma$, push it forward according to all learned OT maps, and classify again based on this new sample. LOT resample x10 refers to this technique, where we resample from $\sigma$ 10 times and average the predicted probabilities to make a classification. We see that this technique achieves the highest test metrics, although the LOT classifier without resampling performs similarly.

    \bibliographystyle{plain}
    \bibliography{references}

\end{document}